\PassOptionsToPackage{table}{xcolor}
\documentclass[11pt]{article}

\usepackage[preprint]{acl}
\usepackage{times}
\usepackage{latexsym}
\usepackage[T1]{fontenc}
\usepackage[utf8]{inputenc}
\usepackage{microtype}
\usepackage{inconsolata}
\usepackage{graphicx}
\usepackage{booktabs}
\usepackage{multirow}
\usepackage{tabularx}
\usepackage{amssymb}
\usepackage{pifont}
\usepackage{enumitem}
\usepackage[table]{xcolor}
\usepackage{pifont}

\newcommand{\rwcheck}{\textcolor{green!55!black}{\ding{51}}}
\newcommand{\rwxmark}{\textcolor{red!70!black}{\ding{55}}}
\usepackage{array}
\usepackage{amsmath}
\usepackage{amssymb}
\usepackage{amsthm}
\usepackage{xcolor}
\usepackage{enumitem}
\usepackage{algorithm}
\usepackage{algpseudocode}
\usepackage{placeins}
\usepackage{url}
\usepackage{times}
\usepackage{hyperref}
\usepackage{xspace}
\usepackage[most]{tcolorbox}

\hypersetup{
  colorlinks=true,
  citecolor=blue!55!black,
  linkcolor=blue!55!black,
  urlcolor=blue!55!black
}
\usepackage[table]{xcolor}
\definecolor{resultgreen}{RGB}{225,243,228}
\newcolumntype{Y}{>{\raggedright\arraybackslash}X}
\newcommand{\sys}{\textsc{MemOnDemand}\xspace}
\newcommand{\dataset}{\textsc{EnterpriseRAG}\xspace}

\newtheorem{proposition}{Proposition}

\newtcolorbox{promptbox}[2][blue]{
  colback=#1!3,
  colframe=#1!55!black,
  coltitle=white,
  colbacktitle=#1!55!black,
  title=\textbf{#2},
  fonttitle=\small,
  fontupper=\footnotesize,
  boxrule=0.7pt,
  arc=1.5pt,
  left=5pt,right=5pt,top=4pt,bottom=4pt
}

\title{\sys: Memory Management for 1.1B-Token Enterprise Data}

\author{
{\normalfont Xinyuan Song$^{1,2}$ \quad
Bowen Zhu$^{1}$} \quad
{\normalfont Hasibul Haque$^{1}$ \quad
Liang Zhao$^{1,2}$} \\
{\normalfont $^{1}$Causal Dynamics Lab, USA \quad
$^{2}$Emory University, USA \quad
}
}

\begin{document}
\maketitle

\begin{abstract}
Enterprise repositories are large, heterogeneous, and continuously updated, making retrieval difficult when efficient access, source-faithful evidence, and cross-query adaptation must be supported together. Enterprise memory extends retrieval beyond the model context, but existing systems do not jointly address collection-specific hierarchy construction, low-cost routing, detailed evidence loading, and workload-aware memory updates at this scale. We introduce \sys, short for \emph{On-Demand Memory}, a memory management system with three coordinated mechanisms: a dynamic multi-level hierarchy that determines the abstraction structure and depth for each collection, dual memory at every hierarchy level that separates distilled routing from detailed evidence, and on-demand memory promotion that updates node priority under a bounded active-state budget. On EnterpriseRAG-Bench, \sys outperforms the strongest published LB\#1 result at every evaluated scale from 10M tokens through the complete 618M-token collection, with gains of 14.53\% at 10M and 2.02\% at 618M. We further scale \sys to an expanded 1.14B-token collection, demonstrating billion-token memory management while maintaining competitive answer quality. Results on FinanceBench, HotpotQA, and FRAMES further show strong performance across financial, multi-hop, and fact-retrieval settings. Together, these results establish \sys as an accurate, efficient, and scalable memory solution for very large enterprise repositories across data scales, domains, and evidence requirements.
Our code: \url{https://github.com/xfab-xinyuansong/MemOnDemand} and 1.14B-token dataset: \url{https://huggingface.co/datasets/xsong69/enterpriseRAG-extension} are publicly available.
\end{abstract}

\section{Introduction}
\label{sec:introduction}

Enterprise repositories contain policies, contracts, email, tables, tickets, code, and records produced by many teams and software systems. In large organizations, these repositories can reach hundreds of millions or billions of searchable items and continue to grow as new records, revisions, and derived data are added~\cite{chang2006bigtable,oneil1996lsm,subramanya2019diskann,wang2021milvus,armbrust2021lakehouse}. Their contents differ in format, scope, and authority: a policy may define a general rule, a contract may specify an exception, an email may record a later decision, and a ticket or table may contain the identifier needed to apply that decision~\cite{armbrust2021lakehouse,codd1970relational}. Their structure also changes with departments, data sources, access patterns, and document versions~\cite{oneil1996lsm,chaudhuri2007selftuning}. Enterprise retrieval must therefore search a large, heterogeneous, and changing collection while preserving exact records and traceable source identities. Approximate similarity indices make large candidate spaces searchable~\cite{malkov2020hnsw,subramanya2019diskann,wang2021milvus}, but they do not define how records should be organized, represented, loaded for answering, or maintained across queries.

\begin{figure*}[!ht]
  \centering
  \includegraphics[width=\textwidth]{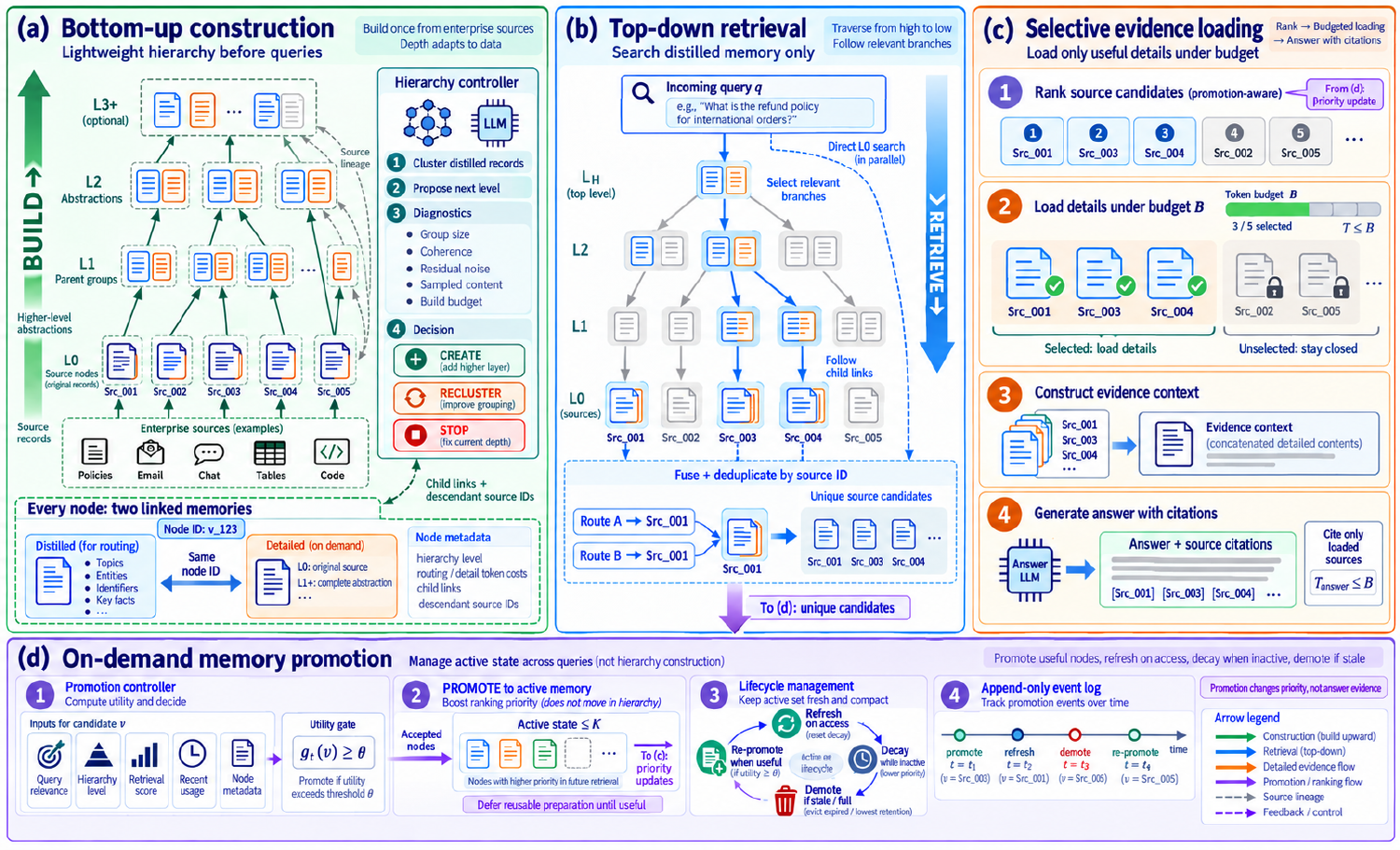}

  \caption{\textbf{\sys overview.}
  (a) \textbf{Bottom-up construction:} A controller builds a data-dependent hierarchy from L0 sources using \textsc{Create}, \textsc{Recluster}, and \textsc{Stop}, with each node linking distilled and detailed memory.
  (b) \textbf{Top-down retrieval:} Queries traverse distilled hierarchy nodes and search L0 directly; both routes are fused and deduplicated by source ID.
  (c) \textbf{Selective evidence loading:} Promotion-aware ranking selects detailed L0 evidence under budget $B$, and answers cite only loaded sources.
  (d) \textbf{On-demand promotion:} An active set of size at most $K$ prioritizes useful nodes through promotion, refresh, decay, and demotion, with transitions logged.}
  \label{fig:system}
  
\end{figure*}

Enterprise memory makes information from large repositories available beyond the model context and reusable across queries~\cite{lewis2020rag,wang2023longmem}. Existing systems retrieve interaction histories~\cite{zhong2023memorybank}, derive higher-level reflections~\cite{park2023generative}, manage active and external memory tiers~\cite{packer2023memgpt}, and support memory extraction, organization, updating, and forgetting~\cite{chhikara2025mem0,xu2025amem,kang2025memoryos,li2025memos}. However, enterprise memory must scale to hundreds of millions or billions of items while preserving source identity, update consistency, and auditable evidence use. Compressed memories may omit answer-critical details~\cite{jiang2023llmlingua,jiang2024longllmlingua,xu2024recomp,packer2023memgpt}, stored memories may become stale~\cite{wang2023longmem}, and most agent-memory systems do not target source-level citation over large multi-source repositories~\cite{park2023generative,zhong2023memorybank,chhikara2025mem0,xu2025amem}. Long-context models also remain sensitive to irrelevant context~\cite{liu2024lost,hsieh2024ruler}. Enterprise memory must therefore support large-scale retrieval and reuse without allowing compressed or stale memories to replace authoritative sources.

The heterogeneous structure of enterprise data motivates hierarchical memory, where high-level records support routing and lower-level records retain source-specific evidence. Prior work studies recursive, graph-based, and multi-level retrieval~\cite{sarthi2024raptor,edge2024graphrag,gutierrez2024hipporag,qian2024memorag,sun2025hmem}. However, enterprise collections differ across domains and tenants, so a fixed hierarchy or predefined taxonomy cannot fit all settings. Enterprise memory should instead infer both the hierarchy and its depth from each collection, while preserving direct source-leaf access when higher-level abstractions omit relevant details~\cite{subramanya2019diskann,indyk2023worstcase}.

Repository scale also requires separating memory for routing from memory for evidence. Detailed memory preserves source-specific content but is expensive to search and load, while compressed or compact memory reduces cost but may omit answer-critical details~\cite{jiang2023llmlingua,jiang2024longllmlingua,xu2024recomp,wang2023longmem,packer2023memgpt,qian2024memorag}. Enterprise retrieval therefore needs dual representations: distilled memory for efficient routing and detailed memory for generation and citation. Both must resolve to the same source ID so that compressed records guide access without replacing the authoritative source~\cite{codd1970relational,abadi2007materialization}.

Enterprise memory must also adapt to changing data and workloads. Eagerly preparing all representations wastes computation and storage on sources that may never be used, while adaptive data systems update access structures according to observed demand~\cite{idreos2007cracking,megiddo2003arc,chaudhuri2007selftuning,oneil1996lsm}. Enterprise memory should therefore promote useful sources on demand, refresh repeatedly accessed records, and demote stale or low-value ones under a fixed capacity, without bypassing source validation or evidence-loading constraints.

We introduce \sys to address these three problems through a unified memory management design. First, \emph{dynamic multi-level hierarchy} determines both the abstraction structure and the number of levels for each collection, allowing the same system to adapt across domains and tenants while retaining direct access to L0 source nodes. Second, \emph{dual memory at each hierarchy level} separates retrieval from answering: distilled memory supports efficient routing, while selected detailed memory provides the evidence used for generation and citation. Third, \emph{on-demand memory promotion} adjusts node priority from observed use, refreshes repeatedly accessed nodes, and demotes stale or low-value nodes under a bounded active-state budget. Together, these mechanisms improve retrieval, ranking, and evidence selection by coordinating hierarchy navigation, compact routing, detailed evidence loading, and cross-query adaptation within one memory manager.

Figure~\ref{fig:system} illustrates the three mechanisms and their interaction. We evaluate \sys on EnterpriseRAG-Bench, a large-scale enterprise retrieval benchmark with approximately 500,000 documents~\cite{sun2026enterpriserag}. Against LB\#1, the strongest published solution on this benchmark, \sys improves Combined by 14.53\% at 10M source tokens and remains 2.02\% higher on the complete 618M-token collection. We further extend evaluation to a 1.14B-token collection, demonstrating billion-token scalability while maintaining strong answer quality. Dual-memory selective loading reduces answer-input tokens by 70.2\% relative to detailed-only loading, while the persisted hierarchy becomes query-ready in 1.46 seconds rather than the estimated 593.5 seconds required for eager preparation. Results on FinanceBench, HotpotQA, and FRAMES further show strong performance across financial, multi-hop, and fact-retrieval settings. Together, these results establish \sys as an accurate, efficient, and scalable memory solution for very large enterprise repositories.

The main contributions of this work are summarized as follows:

\begin{table*}[!ht]
\centering
\caption{\textbf{Comparison with closely related retrieval and memory systems.}
We compare dynamic hierarchy construction, separation of compact routing and detailed evidence (dual memory), selective evidence loading, persistent cross-query priorities, and separate budgets for answer context ($B$) and active state ($K$).}

\label{tab:related_work_comparison}
\begingroup
\small
\setlength{\tabcolsep}{5.5pt}
\renewcommand{\arraystretch}{1.12}
\begin{tabular}{@{}lcccccc@{}}
\toprule
\textbf{Method} &
\textbf{\shortstack{Hierarchical\\organization}} &
\textbf{\shortstack{Dynamic\\hierarchy}} &
\textbf{\shortstack{Dual\\memory}} &
\textbf{\shortstack{Selective\\loading}} &
\textbf{\shortstack{Persistent\\priority}} &
\textbf{\shortstack{Separate\\$B/K$}} \\
\midrule

RAPTOR~\citep{sarthi2024raptor}
& \rwcheck & \rwxmark & \rwxmark & \rwcheck & \rwxmark & \rwxmark \\

HiRAG~\citep{huang2025hirag}
& \rwcheck & \rwcheck & \rwxmark & \rwcheck & \rwxmark & \rwxmark \\

H-MEM~\citep{sun2025hmem}
& \rwcheck & \rwxmark & \rwxmark & \rwcheck & \rwcheck & \rwxmark \\

HippoRAG~2~\citep{gutierrez2025hipporag2}
& \rwcheck & \rwxmark & \rwxmark & \rwcheck & \rwxmark & \rwxmark \\

\midrule

MemoRAG~\citep{qian2024memorag}
& \rwxmark & \rwxmark & \rwcheck & \rwcheck & \rwxmark & \rwxmark \\

REFRAG~\citep{lin2025refrag}
& \rwxmark & \rwxmark & \rwcheck & \rwcheck & \rwxmark & \rwxmark \\

\midrule

A-MEM~\citep{xu2025amem}
& \rwxmark & \rwxmark & \rwxmark & \rwxmark & \rwcheck & \rwxmark \\

Mem0~\citep{chhikara2025mem0}
& \rwxmark & \rwxmark & \rwxmark & \rwxmark & \rwcheck & \rwxmark \\

MemoryOS~\citep{kang2025memoryos}
& \rwcheck & \rwxmark & \rwxmark & \rwxmark & \rwcheck & \rwxmark \\

MemOS~\citep{li2025memos}
& \rwcheck & \rwxmark & \rwxmark & \rwcheck & \rwcheck & \rwxmark \\

AgeMem~\citep{yu2026agemem}
& \rwxmark & \rwxmark & \rwxmark & \rwcheck & \rwcheck & \rwxmark \\

\midrule

\rowcolor{blue!9}
\textbf{\sys (ours)}
& \rwcheck
& \rwcheck
& \rwcheck
& \rwcheck
& \rwcheck
& \rwcheck \\

\bottomrule
\end{tabular}
\endgroup

\end{table*}

\begin{itemize}[left = 0em]
    \item We introduce \sys, a memory management system for large-scale enterprise retrieval that jointly manages hierarchical organization, retrieval representation, evidence loading, and cross-query memory state.
    \item We develop three coordinated mechanisms: a dynamic multi-level hierarchy that adapts its depth and abstractions to each collection, dual memory at every hierarchy level that separates efficient routing from source-faithful answering, and on-demand memory promotion that adapts node priority to changing workloads under a bounded active-state budget.
    \item We demonstrate that a managed memory system can operate successfully at ultra-large scale: on the complete 618M-token EnterpriseRAG memory, \sys reaches 70.24\% Combined, 2.02\% above the published LB\#1 reference. Across scales from 10M tokens to 618M and on three external benchmarks, it also reduces answer-input cost, improves query readiness, and maintains strong performance on FinanceBench, HotpotQA, and FRAMES.
\end{itemize}

\section{Related Work}

\label{sec:related}

\paragraph{Retrieval and Hierarchical Organization.}
Retrieval-augmented generation~\cite{lewis2020rag} integrates external search with generation, supported by dense retrieval~\cite{karpukhin2020dpr} and scalable indexing~\cite{malkov2020hnsw,subramanya2019diskann,chen2021spann}. Adaptive-RAG~\cite{jeong2024adaptiverag} and Self-RAG~\cite{asai2024selfrag} enable query-dependent retrieval, while RAPTOR~\cite{sarthi2024raptor}, GraphRAG~\cite{edge2024graphrag}, HippoRAG~\cite{gutierrez2024hipporag}, H-MEM~\cite{sun2025hmem}, and MemoRAG~\cite{qian2024memorag} support hierarchical information access. Unlike these approaches, \sys jointly determines collection-specific hierarchy structure while retaining direct source-level evidence access.

\paragraph{Compression and Persistent Memory.}
Limited effective context~\cite{liu2024lost,hsieh2024ruler} motivates compression methods including LLMLingua~\cite{jiang2023llmlingua}, LongLLMLingua~\cite{jiang2024longllmlingua}, and RECOMP~\cite{xu2024recomp}. Persistent-memory systems~\cite{park2023generative,zhong2023memorybank,packer2023memgpt,wang2023longmem} retain information across interactions, with Mem0~\cite{chhikara2025mem0}, A-MEM~\cite{xu2025amem}, MemoryOS~\cite{kang2025memoryos}, and MemOS~\cite{li2025memos} advancing memory organization and lifecycle management. \sys unifies these directions through compact routing memory, detailed evidence, and persistent retrieval priority.

\paragraph{Database Physical Design.}
Database systems separate logical records from physical access structures~\cite{codd1970relational}. Adaptive indexing~\cite{idreos2007cracking}, deferred materialization~\cite{abadi2007materialization}, cache replacement~\cite{megiddo2003arc}, and self-tuning~\cite{chaudhuri2007selftuning} adapt access to evolving workloads. Similarly, \sys separates memory organization, evidence loading, and query-driven priority updates.

\paragraph{Relationship to Prior Work.}
Table~\ref{tab:related_work_comparison} shows how \sys integrates dynamic hierarchy construction, dual memory, budgeted evidence loading, and persistent retrieval priorities. By separating organization, evidence cost, and reuse, \sys improves efficiency while scaling to large collections (Tables~\ref{tab:packing} and~\ref{tab:scaling}).

\section{\sys}
\label{sec:system}

\subsection{Problem Formulation}
\label{sec:problem}

We consider retrieval over a large enterprise collection \(\mathcal{D}=\{d_i\}_{i=1}^{n}\) for a query sequence \(\mathcal{Q}=(q_1,q_2,\ldots)\). Each source retains its original content and identity, while the system may construct auxiliary representations and retrieval states~\cite{codd1970relational,abadi2007materialization}.

For each query \(q_t\), the system selects an evidence set \(E_t\subseteq\mathcal{D}\) and generates an answer \(a_t\) with citations \(C_t\), subject to
\begin{equation}
C_t\subseteq\mathrm{ID}(E_t),\qquad T(q_t,E_t)\le B,
\label{eq:problem_constraints}
\end{equation}
where \(T(q_t,E_t)\) is the answer-input token count and \(B\) is the answer-input budget. These constraints ensure that citations refer only to loaded evidence within the available context.

The objective is accurate and efficient retrieval as collections grow and evolve. This requires adaptive hierarchy construction, compact routing memory, detailed source evidence, and persistent retrieval priorities under a separate active-memory capacity \(K\). Unlike \(B\), which limits per-query evidence, \(K\) bounds cross-query memory state. Simply increasing retrieved context is insufficient, as models may fail to identify decisive evidence in long inputs~\cite{liu2024lost,hsieh2024ruler}.

\subsection{System Contract and Execution}

\sys maintains three forms of state aligned with its three mechanisms: a dynamic hierarchy over memory nodes, distilled and detailed memory at each node, and cross-query promotion state. Each node records its hierarchy level, two memory representations, token cost, child relations, and current promotion status. This design follows the database principle that multiple access paths and derived structures should remain tied to the same underlying record~\cite{codd1970relational,abadi2007materialization}.

Given a query, \sys performs coarse-to-fine retrieval over distilled memory, starting from the highest available level and descending from \(L_{\ell}\) to \(L_{\ell-1}\), \(L_{\ell-2}\), and finally L0~\cite{sarthi2024raptor,edge2024graphrag}. Sentence embeddings support efficient search at each level~\cite{reimers2019sbert}. The retrieved L0 candidates are then ranked with promotion signals, after which detailed memory is loaded under the answer-input budget \(B\). Only the selected detailed memory is provided to the answer model, keeping generation independent from the retrieval and storage components~\cite{lewis2020rag}.

This execution order assigns a distinct role to each mechanism: the dynamic hierarchy organizes and narrows retrieval, dual memory separates routing from evidence, and on-demand promotion adapts candidate priority across queries. Algorithm~\ref{alg:online} summarizes the complete pipeline, and the following subsections describe each mechanism.

\subsection{Dynamic Multi-Level Hierarchy}
\label{sec:dynamic}

\sys organizes enterprise memory into a data-dependent hierarchy. L0 contains original source records and their distilled routing representations; L1 groups related L0 records, while L2 summarizes related L1 records. Higher levels are constructed recursively when supported by the collection, without fixing hierarchy depth or grouping structure in advance.

Starting from \(L_\ell\), \sys clusters distilled representations to propose \(L_{\ell+1}\). Each proposal includes group-size statistics, cluster coherence, residual noise, sampled content, and remaining construction budget. An LLM controller determines the next operation:
\begin{equation}
a_\ell=\pi_{\mathrm{LLM}}(S_\ell)\in\{\textsc{Create},\textsc{Recluster},\textsc{Stop}\},
\label{eq:hierarchy_control}
\end{equation}
where \(S_\ell\) summarizes the clustering proposal and construction budget. \textsc{Create} accepts the grouping and constructs \(L_{\ell+1}\), \textsc{Recluster} revises it for reevaluation, and \textsc{Stop} terminates construction. This enables collection-specific hierarchy depth and structure~\cite{sarthi2024raptor,edge2024graphrag}. Each parent retains links to its descendant L0 sources. GPT-5.4 mini generates L0 distilled records, while GPT-5.4 controls higher-level construction. Figure~\ref{fig:system}(a) illustrates the process, Table~\ref{tab:hierarchy} reports hierarchy sizes, and Algorithm~\ref{alg:hierarchy} specifies construction.

At query time, \sys retrieves relevant distilled records from the highest available level \(L_\ell\). A routing controller evaluates whether further refinement is needed; if so, retrieval follows selected child links to \(L_{\ell-1}\) and repeats recursively. The selected branches ultimately resolve to L0 source candidates, which provide detailed evidence for ranking, generation, and citation. This separates coarse semantic routing from source-level evidence access.

\subsection{Dual Memory at Each Hierarchy Level}
\label{sec:representations}

Each hierarchy node maintains two representations under a shared identity: \emph{distilled memory} for compact retrieval and \emph{detailed memory} for richer information access. Distilled memory preserves key topics, entities, identifiers, and decision-relevant facts, while detailed memory retains original source content at L0 or complete child-derived abstractions at higher levels. This separation enables efficient routing without repeatedly loading full content~\cite{jiang2023llmlingua,jiang2024longllmlingua,xu2024recomp}.

For a node \(v\) at level \(L_\ell\), \sys stores
\begin{equation}
v=(\ell_v,b_v,d_v,\tau_v^b,\tau_v^d),
\label{eq:dual_memory}
\end{equation}
where \(\ell_v\) denotes its level, \(b_v\) and \(d_v\) are distilled and detailed memories, and \(\tau_v^b\) and \(\tau_v^d\) are their token costs. Both representations are indexed separately but linked to the same node. During top-down retrieval, \sys searches distilled memory and follows selected child links without loading detailed memory.

At L0, candidates from hierarchical and direct retrieval are merged by source identity, with their retrieval signals fused for ranking. Only the highest-ranked detailed memories fitting budget \(B\) are loaded for answer generation and citation. This deferred loading reduces unnecessary token consumption~\cite{abadi2007materialization}. Figure~\ref{fig:system}(b--c) illustrates the retrieval and evidence-loading process, while Algorithm~\ref{alg:selection} specifies candidate fusion, ranking, and budgeted selection.

\subsection{On-Demand Memory Promotion}
\label{sec:promotion}

Enterprise workloads evolve, making fixed memory priorities ineffective~\cite{chaudhuri2007selftuning,oneil1996lsm}. Preparing every node in advance wastes resources on rarely accessed records~\cite{abadi2007materialization,idreos2007cracking}. \sys instead adapts retrieval priorities to observed query demand through on-demand promotion~\cite{idreos2007cracking,megiddo2003arc,chaudhuri2007selftuning}.

For each candidate node \(v\) at query step \(t\), a promotion controller computes \(g_t(v)\in[0,1]\) from the query, hierarchy level, retrieval scores, node metadata, and recent usage. Nodes satisfying \(g_t(v)\ge\theta\) are promoted, increasing their retrieval priority and updating their stored promotion state \(s_t(v)\). Repeated access refreshes this state, while inactivity reduces its retention value~\cite{megiddo2003arc,zhong2023memorybank}. Promotion affects ranking and cross-query reuse, whereas detailed evidence loading remains separately constrained by budget \(B\)~\cite{abadi2007materialization,packer2023memgpt}.

To bound persistent state, \sys maintains at most \(K\) active nodes, with retention score
\begin{equation}
\rho_t(v)=\max\left(0,s_t(v)-\frac{t-t_{\mathrm{last}}(v)}{\tau}\right),
\label{eq:retention}
\end{equation}
where \(s_t(v)\) is the stored promotion score, \(t_{\mathrm{last}}(v)\) is the latest access time, and \(\tau\) controls decay. When capacity is exceeded, expired nodes are demoted first, followed by those with the lowest retention scores until \(|\operatorname{active}(\mathbf{z}_t)|\le K\). Demoted nodes remain eligible for future promotion. All promotion, refresh, and demotion events are recorded in an append-only log.

Figure~\ref{fig:system}(d) illustrates this lifecycle. By separating bounded promotion state from budgeted evidence loading, \sys adapts retrieval priorities across queries without rebuilding the hierarchy or treating promoted nodes as answer evidence.

\begin{algorithm}[!ht]
\caption{\sys construction and query execution}
\label{alg:online}
\small
\begin{algorithmic}[1]
\Require sources \(\mathcal D\), queries \(\mathcal Q\), answer-input budget \(B\)
\State build L0 memory and dynamically construct higher abstraction levels
\For{\(q_t\in\mathcal Q\)}
  \State retrieve distilled memory from the highest level to L0
  \State score candidate nodes and apply on-demand promotion
  \State refresh reused nodes and demote nodes whose retention scores expire
  \State rank L0 candidates and load detailed memory until budget \(B\) is reached
  \State generate the answer from loaded detailed memory and log memory-state changes
\EndFor
\State \Return answers, retrieved evidence, updated memory state, and audit events
\end{algorithmic}
\end{algorithm}

\paragraph{System properties.}
Appendix~\ref{app:properties} describes the retrieval, context-accounting, and memory-update properties of \sys. These properties specify how the three mechanisms interact during execution rather than introducing a separate theoretical objective.

\section{Experimental Setup}
\label{sec:experiments}

Our primary evaluation uses EnterpriseRAG-Bench (\dataset), a synthetic benchmark for enterprise retrieval over heterogeneous sources, multi-source evidence, and citation-based answering~\cite{sun2026enterpriserag}. We evaluate seven collection sizes containing 10M, 20M, 60M, 100M, 150M, 250M, and the complete 618M source tokens. Each system answers the same 500 questions at every scale and returns both an answer and the source IDs used to support it. We additionally evaluate on FinanceBench for financial-document question answering~\cite{islam2023financebench}, HotpotQA for multi-hop question answering~\cite{yang2018hotpotqa}, and FRAMES for fact retrieval and reasoning over multiple sources~\cite{krishna2024frames}.

\paragraph{EnterpriseRAG-Bench metrics.} Combined is the primary benchmark score and summarizes answer and evidence quality. Combined, Correct, Complete, and Document Recall are reported as percentages, with \(\Delta\) Combined denoting the absolute difference from LB\#1 in \%. Evidence F1 and Invalid Document Ratio (InvDoc) are reported on the \([0,1]\) scale, where higher Evidence F1 and lower InvDoc are better. Promo and Demote report the total numbers of accepted promotion and demotion transitions across the complete query run. Table~\ref{tab:scaling} also includes LB\#1, the strongest published \dataset result under the same 500-question protocol and a collection of approximately 500,000 documents~\cite{sun2026enterpriserag}.
Additional implementation details, including model assignment, retrieval settings, evaluation rules, failure handling, indexing, and caching, are provided in Appendix~\ref{app:implementation} and summarized in Table~\ref{tab:implementation}.

\paragraph{Scaling Beyond the Official Collection.}
To evaluate billion-token scalability, we extend the unchanged 618M-token EnterpriseRAG-Bench collection to 1.14B tokens using 522M tokens of project-conditioned synthetic records. The extension preserves source distributions and uses disjoint entities and identifiers; it is reported as \emph{Expanded}, distinct from the official \emph{FULL} tier. Construction details and source composition are provided in Appendix~\ref{sec:data_extension}.

\begin{table*}[!ht]
\centering
\setlength{\tabcolsep}{4.2pt}
\resizebox{0.95\textwidth}{!}{%
\begin{tabular}{lrrrrrrrrr}
\toprule
Tier & Combined \(\uparrow\) & \(\Delta\) Combined & Correct \(\uparrow\) & Complete \(\uparrow\) & DocRcl \(\uparrow\) & Evid. F1 \(\uparrow\) & InvDoc \(\downarrow\) & Promo & Demote \\
 & (\%) & (\%) & (\%) & (\%) & (\%) & & & & \\
\midrule
LB\#1 & \textbf{68.22} & -- & 81.60 & 72.86 & 79.02 & -- & 0.4700 & -- & -- \\
\midrule
10M & \cellcolor{resultgreen}\textbf{82.75} & \cellcolor{resultgreen}+14.53 & 91.40 & 85.83 & 92.71 & 0.928 & 0.1200 & 790 & 762 \\
20M & \cellcolor{resultgreen}\textbf{81.66} & \cellcolor{resultgreen}+13.44 & 90.80 & 84.08 & 90.04 & 0.900 & 0.1700 & 811 & 784 \\
60M & \cellcolor{resultgreen}\textbf{73.09} & \cellcolor{resultgreen}+4.87 & 80.45 & 76.97 & 80.30 & 0.780 & 0.4300 & 864 & 839 \\
100M & \cellcolor{resultgreen}\textbf{76.52} & \cellcolor{resultgreen}+8.30 & 85.80 & 79.43 & 83.51 & 0.833 & 0.3200 & 861 & 839 \\
150M & \cellcolor{resultgreen}\textbf{76.15} & \cellcolor{resultgreen}+7.93 & 84.00 & 78.87 & 82.09 & 0.820 & 0.3400 & 901 & 877 \\
250M & \cellcolor{resultgreen}\textbf{71.48} & \cellcolor{resultgreen}+3.26 & 79.20 & 75.59 & 78.65 & 0.783 & 0.4300 & 905 & 881 \\
FULL (618M) & \cellcolor{resultgreen}\textbf{70.24} & \cellcolor{resultgreen}+2.02 & 77.78 & 75.28 & 75.02 & 0.727 & 0.5100 & 923 & 900 \\
Expanded (1.14B) & \textbf{66.13} & -2.09 & 71.28 & 72.22 & 68.54 & 0.650 & 0.7100 & 933 & 912 \\
\bottomrule
\end{tabular}
}

\caption{\textbf{Primary scaling results.} LB\#1 is the published \dataset baseline; \(\Delta\) Combined reports absolute percentage-point differences. Promo and Demote count state transitions over 500 queries. FULL is the official 618M-token collection, while Expanded extends it to 1.14B tokens (Appendix~\ref{sec:data_extension}).}
\label{tab:scaling}

\end{table*}

\section{Results}
\label{sec:results}

\subsection{Scaling to the Full 618M-Token Collection}
\label{sec:scaling}

Table~\ref{tab:scaling} shows that \sys consistently outperforms LB\#1, the strongest published \dataset baseline, across all seven official corpus tiers. Combined improvements reach 14.53\% at 10M tokens and remain positive through the full 618M-token collection (+2.02 points). Despite increasing retrieval difficulty at larger scales, \sys maintains strong answer correctness and completeness, demonstrating the effectiveness of hierarchical retrieval, selective evidence loading, and adaptive memory promotion. Figure~\ref{fig:scaling} illustrates these trends.

We further evaluate \sys on an Expanded 1.14B-token collection, adding 522M tokens of project-conditioned synthetic records to the unchanged official corpus (Appendix~\ref{sec:data_extension}). \sys achieves 66.13\% Combined, with 71.28\% Correct and 72.22\% Complete. Although performance declines at this scale, the results demonstrate that \sys can operate on billion-token collections while maintaining useful answer quality. Expanded serves as a separate scalability stress test rather than an official benchmark tier.

\begin{table*}[!ht]
\centering
\setlength{\tabcolsep}{2pt}
\resizebox{0.8\linewidth}{!}{
\begin{tabularx}{\linewidth}{@{}lYrrr@{}}
\toprule
Scale & Memory provided for answering & \shortstack{Answer-input\\tokens/query} & \shortstack{Change\\(\%)} & \shortstack{Combined (\%)\\\(\uparrow\)} \\
\midrule
10M & Detailed only & 46,994 & -- & 82.68 \\
10M & Both representations & 49,174 & +4.6\% & 70.60 \\
10M & Distilled only & 2,963 & -93.7\% & 51.08 \\
\rowcolor{resultgreen}
10M & \textbf{Selective detailed memory (\sys)} & \textbf{16,944} & \textbf{-63.9\%} & \textbf{82.75} \\
\midrule
20M & Detailed only & 45,784 & -- & 77.91 \\
20M & Both representations & 51,078 & +11.6\% & 68.22 \\
20M & Distilled only & 3,082 & -93.3\% & 49.76 \\
\rowcolor{resultgreen}
20M & \textbf{Selective detailed memory (\sys)} & \textbf{13,629} & \textbf{-70.2\%} & \textbf{81.66} \\
\bottomrule
\end{tabularx}
}

\caption{\textbf{Token consumption across memory policies at 10M and 20M.} Changes are relative to detailed-only loading. Selective-detail quality uses revised prompts, while token counts follow matched packing traces.}
\label{tab:packing}

\end{table*}

\subsection{Token Consumption Analysis}
\label{sec:representation}

To assess dual-memory efficiency, we fix retrieval traces and vary the memory representations provided for answering. Distilled memory supports retrieval and ranking across all settings, while detailed evidence is loaded according to each policy. Selective-detail quality is evaluated using the revised evidence-grounded prompt.

Table~\ref{tab:packing} shows that \sys reduces answer-input tokens by 63.9\% and 70.2\% at 10M and 20M tokens, respectively, while achieving the highest Combined scores of 82.75\% and 81.66\%. Distilled-only loading sacrifices answer quality, whereas loading both representations increases token costs without improving performance. These results demonstrate that selective detailed-memory loading preserves answer quality while substantially reducing context consumption.

Table~\ref{tab:runtime_tokens} in Appendix~\ref{app:additional_results} further decomposes online token use and shows that answer input is the main source of token cost.

\subsection{Comparison with Direct Retrieval Strategies}
\label{sec:frontier}

Table~\ref{tab:frontier} compares \sys with BM25~\cite{robertson2009bm25}, flat dense RAG~\cite{karpukhin2020dpr}, detailed-only hierarchical retrieval, and direct long-context packing~\cite{liu2024lost}. All methods share the same 500 questions, answer/evaluation models, and embedding backend; \sys uses the revised evidence-grounded answer prompt.

Across both scales, \sys achieves the highest Combined, Correct, and Document Recall scores. At 20M, it improves Combined over flat dense RAG by 9.98\% while using substantially fewer tokens than detailed-only hierarchy retrieval and long-context packing. Although flat dense RAG uses less context, its retrieval coverage and answer quality are lower. Figure~\ref{fig:frontier} further illustrates \sys's favorable quality--context trade-off.

\begin{table*}[!ht]
\centering
\resizebox{0.85\textwidth}{!}{%
\begin{tabular}{llrrrr}
\toprule
Scale & Strategy & Answer-input tokens/query & Combined (\%) \(\uparrow\) & Correct (\%) \(\uparrow\) & DocRcl (\%) \(\uparrow\) \\
\midrule
\multirow{5}{*}{10M}
& Flat dense RAG & 6,001 & 74.48 & 74.6 & 52.91 \\
& BM25 RAG, no promotion & 5,726 & 29.68 & 34.4 & 24.43 \\
& Hierarchical, detailed only & 46,994 & 46.80 & 52.2 & 70.03 \\
& Direct long-context packing & 93,318 & 44.44 & 49.4 & 78.05 \\
\rowcolor{resultgreen}
& \textbf{\sys} & \textbf{16,944} & \textbf{82.75} & \textbf{91.40} & \textbf{92.71} \\
\midrule
\multirow{5}{*}{20M}
& Flat dense RAG & 6,018 & 71.68 & 71.8 & 48.09 \\
& BM25 RAG, no promotion & 5,957 & 30.36 & 35.0 & 23.11 \\
& Hierarchical, detailed only & 45,784 & 43.16 & 48.2 & 66.87 \\
& Direct long-context packing & 90,832 & 42.76 & 47.8 & 75.68 \\
\rowcolor{resultgreen}
& \textbf{\sys} & \textbf{13,629} & \textbf{81.66} & \textbf{90.80} & \textbf{90.04} \\
\bottomrule
\end{tabular}
}

\caption{\textbf{Comparison with direct retrieval strategies.} All methods share 500 questions, evaluation models, and embeddings; \sys uses the revised answer prompt.}
\label{tab:frontier}

\end{table*}

\subsection{Component Ablations}
\label{sec:ablation}

We ablate three design choices at 10M tokens: limiting hierarchy traversal to one step, disabling promotion and demotion, and expanding the candidate pool from 15 to 25. We also compare against flat dense retrieval, global source-leaf search without promotion, and the published LB\#1 baseline to assess the contributions of hierarchical navigation, adaptive prioritization, and evidence selection.

\begin{table*}[!ht]
\centering
\footnotesize
\resizebox{0.8\linewidth}{!}{%
\begin{tabular}{lrrrr}
\toprule
Method & DocRcl (\%) \(\uparrow\) & Correct (\%) \(\uparrow\) & Combined (\%) \(\uparrow\) & \(\Delta\) vs. Full (\%) \\
\midrule
\multicolumn{5}{l}{\textit{System variants}} \\
\rowcolor{resultgreen}
\textbf{Full system} & \textbf{92.71} & \textbf{91.40} & \textbf{82.75} & -- \\
One-step navigation & 59.41 & 79.4 & 64.52 & \(-18.23\) \\
No promotion & 59.60 & 82.7 & 64.20 & \(-18.55\) \\
25-candidate pool & 57.44 & 78.4 & 62.64 & \(-20.11\) \\
\midrule
\multicolumn{5}{l}{\textit{Retrieval baselines}} \\
Flat dense retrieval & 52.91 & 74.6 & 74.48 & \(-8.27\) \\
Global source-leaf search, no promotion & 68.25 & 81.2 & 76.68 & \(-6.07\) \\
\midrule
\multicolumn{5}{l}{\textit{Published reference}} \\
LB\#1 & 79.02 & 81.6 & 68.22 & \(-14.53\) \\
\bottomrule
\end{tabular}
}

\caption{\textbf{Ablation and retrieval comparisons at 10M.} Variants modify individual components; $\Delta$ reports the Combined difference from the full system in \%.}
\label{tab:ablation}

\end{table*}

\begin{table*}[!ht]
\centering
\footnotesize
\resizebox{0.7\linewidth}{!}{%
\begin{tabular}{lrrrrr}
\toprule
Benchmark & Questions & DocRcl (\%) \(\uparrow\) & Evid. F1 \(\uparrow\) & Correct (\%) \(\uparrow\) & InvDoc \(\downarrow\) \\
\midrule
FinanceBench & 150 & 80.00 & 0.729 & 66.0 & 0.289 \\
HotpotQA & 500 & 49.90 & 0.549 & 78.4 & 0.320 \\
FRAMES & 500 & 88.76 & 0.617 & 55.0 & 0.439 \\
\bottomrule
\end{tabular}
}

\caption{\textbf{External results.} Retrieval and answer quality for financial QA, multi-hop reasoning, and fact retrieval.}
\label{tab:external}

\end{table*}

Table~\ref{tab:ablation} shows that \sys achieves the highest Combined score (82.75\%) and Document Recall (92.71\%). Limiting traversal, removing promotion, or expanding the candidate pool reduces Combined by 18.23--20.11\%, highlighting the importance of deep retrieval, adaptive prioritization, and focused evidence selection. The full system also outperforms flat dense retrieval, global source-leaf search, and LB\#1, demonstrating the benefits of coordinating all three mechanisms.

We further evaluate the efficiency of on-demand promotion and its activity during online retrieval, with detailed results provided in Appendix~\ref{sec:availability} and Appendix~\ref{sec:runtime_activity}.





\subsection{External Benchmark Evaluation}

\label{sec:promotion_results}

We evaluate \sys on FinanceBench~\cite{islam2023financebench}, HotpotQA~\cite{yang2018hotpotqa}, and FRAMES~\cite{krishna2024frames}, covering financial question answering, multi-hop reasoning, and fact retrieval. Table~\ref{tab:external} shows strong performance across these settings, including 80.00\% Document Recall on FinanceBench, 78.4\% Correct on HotpotQA, and 88.76\% Document Recall on FRAMES. Despite differences in document structure and reasoning complexity, these results demonstrate the applicability of \sys beyond EnterpriseRAG-Bench.

\section{Conclusion}
\label{sec:conclusion}

We present \sys, a memory management system combining a dynamic hierarchy, dual memory, and on-demand promotion. Across EnterpriseRAG-Bench scales from 10M to the full 618M-token collection, \sys consistently outperforms the published LB\#1 baseline. We further demonstrate billion-token scalability on an expanded 1.14B-token collection, while external benchmarks show applicability across diverse retrieval and reasoning tasks. Overall, \sys provides a practical solution for extending LLMs with billion-token memory beyond the model context window.

\section*{Limitations}
\label{sec:limitations}

Our main scaling results and promotion-floor sweeps are based on single runs, so repeated-run variance and confidence intervals remain to be studied. The eager preparation cost is estimated from measured call latency rather than measured through a complete end-to-end run. In addition, Document Recall decreases at the largest scale, and the external benchmarks show that strong source recovery does not always translate into equally strong answer correctness. Future work should evaluate more domains, longer query streams, and broader model and retrieval configurations.

\bibliography{references}

\clearpage
\appendix

\section{Extending the Collection to 1.14B Tokens}
\label{sec:data_extension}

The complete EnterpriseRAG-Bench collection contains 618M source tokens. We retain that official collection unchanged and construct a separate extension for testing memory beyond the benchmark's largest tier. The extension adds approximately 522M tokens, producing a merged collection of 1.14B tokens. We report the official 618M result as \emph{FULL} and the larger collection as \emph{Expanded} throughout; the latter is not presented as an additional official benchmark tier. The released collection, schemas, and data card are available on \url{https://huggingface.co/datasets/xsong69/enterpriseRAG-extension}.

\paragraph{Project-conditioned generation.}
Generating records independently can preserve a source format while losing the shared facts that make cross-source retrieval meaningful. We instead generate each synthetic project in two stages, following the project-centric organization of the benchmark's data-generation procedure~\cite{sun2026enterpriserag}. First, a project scaffold specifies a fictional customer, stable entities, precise facts, a timeline, and a plan for four to eight related records. Second, the same scaffold conditions every record generated for that project. The records can therefore distribute one event across email, chat, issue tracking, meetings, and documents without changing names, dates, or identifiers between sources.

\paragraph{Distribution and separation.}
Generation quotas match the source mixture measured in the official collection: 55.79\% Slack, 23.71\% Gmail, 6.90\% Linear, 4.90\% Google Drive, 2.93\% HubSpot, 1.99\% Fireflies, 1.57\% GitHub, 1.20\% Jira, and 1.01\% Confluence. We introduce 40 fictional customer organizations that do not overlap with the official collection and assign every generated record a disjoint identifier. The merged data retain the same source types and retrieval contract as the 618M collection. Figure~\ref{fig:data_extension_scale} shows both the resulting token growth and the source allocation used for generation.

\begin{figure}[!ht]
  \centering
  \includegraphics[width=0.94\columnwidth]{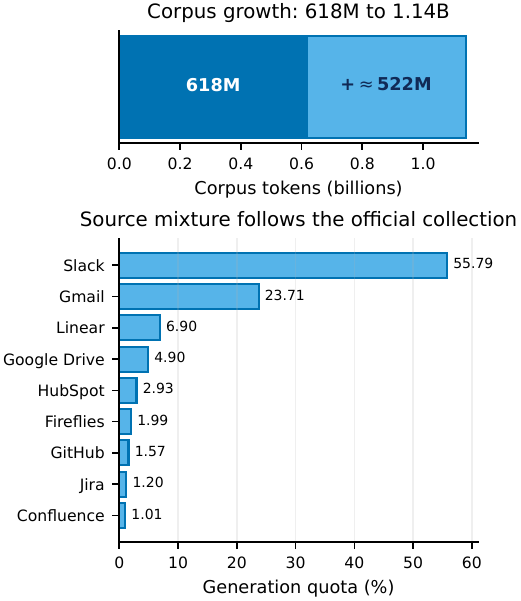}
  \caption{\textbf{Scale and source composition of the extension.} Approximately 522M generated tokens extend the unchanged 618M-token base to 1.14B tokens. Source-specific generation quotas reproduce the empirical source mixture of the official collection.}
  \label{fig:data_extension_scale}

\end{figure}

\paragraph{License and Release.}
The original EnterpriseRAG-Bench dataset is publicly released under the MIT License. Following its project-conditioned generation methodology~\cite{sun2026enterpriserag}, we extend the collection while preserving its source structure and data format. We also publicly release our 522M-token extension under the MIT License to support reproducibility and further research at billion-token scale.

\section{Efficiency of On-Demand Promotion}

\label{sec:availability}

We evaluate whether on-demand promotion reduces the wall time required to prepare reusable parent-level memory. At 20M source tokens, \sys prepares parent representations only when they are promoted by observed query demand, resulting in a measured wall time of 1.46 seconds. In contrast, eagerly preparing all promotable parent representations is estimated to require 593.5 seconds under the measured per-call latency, yielding a \(407\times\) reduction in up-front wall time. This avoids spending computation on parent nodes that may never be reused. Table~\ref{tab:availability} and Figure~\ref{fig:availability} report the comparison.

\begin{table}[!ht]
\centering
\small
\resizebox{0.9\columnwidth}{!}{%
\begin{tabular}{lr}
\toprule
Promotion strategy & Wall time \\
\midrule
\rowcolor{resultgreen}
\textbf{On-demand promotion} & \textbf{1.46 s} \\
Eager promotion of all parent nodes & 593.5 s \\
Wall-time ratio & \textbf{\(407\times\)} \\
\bottomrule
\end{tabular}
}
\caption{\textbf{Preparation wall time under different promotion strategies at 20M.} On-demand promotion avoids eagerly preparing every promotable parent node.}
\label{tab:availability}

\end{table}

\section{Online Use of Promotion}

\label{sec:runtime_activity}

We examine how often on-demand promotion contributes during online retrieval at 20M source tokens. \sys uses 2.9 hierarchy-navigation steps per query on average, with a 95th-percentile depth of three steps, showing that promotion operates within a bounded retrieval process. Promotion is activated for 470 of the 500 questions, or 94.0\% of the evaluation set, indicating that it broadly updates candidate priority and reusable memory state. Together with the no-promotion ablation in Section~\ref{sec:ablation}, these results show that promotion is both frequently used during retrieval and important for final performance. Table~\ref{tab:runtime_activity} reports the detailed activity statistics, and Figure~\ref{fig:external_profile} visualizes the corresponding promotion profile.

\begin{table}[!ht]
\centering
\small
\resizebox{0.9\columnwidth}{!}{%
\begin{tabular}{lr}
\toprule
Online quantity & Value \\
\midrule
Mean navigation steps per query & 2.9 \\
95th-percentile navigation steps & 3 \\
Questions activating promotion & 470 / 500 (94.0\%) \\
\bottomrule
\end{tabular}
}
\caption{\textbf{Promotion activity at 20M.} Statistics cover the complete set of 500 evaluation queries.}
\label{tab:runtime_activity}

\end{table}

\section{Implementation Details}
\label{app:implementation}

\subsection{Model and retrieval configuration.}

GPT-5.4 mini constructs L0 distilled memory and extracts key facts, while GPT-5.4 performs higher-level hierarchy construction, navigation, source selection, answer generation, and \dataset evaluation. All semantic indices use 1,536-dimensional \texttt{text-embedding-3-small} embeddings~\cite{openai2026gpt54mini,openai2026gpt54,openai2026te3small}. Table~\ref{tab:implementation} summarizes the model assignment for each operation.

\paragraph{Evaluation protocol.} We follow the \dataset definitions for answers, supporting evidence, and source IDs~\cite{sun2026enterpriserag}. Each generated answer is compared with the canonical answer, and each returned source ID is checked against the benchmark evidence set. The predefined \texttt{STOP\_INSUFFICIENT} rule is applied during aggregation.

\begin{table}[!ht]
\centering
\resizebox{\columnwidth}{!}{%
\begin{tabular}{ll}
\toprule
Operation & API identifier \\
\midrule
L0 distillation and key facts & \texttt{gpt-5.4-mini} \\
Higher-level decision and abstraction & \texttt{gpt-5.4} \\
Navigation and source selection & \texttt{gpt-5.4} \\
Answer generation & \texttt{gpt-5.4} \\
\dataset evaluation & \texttt{gpt-5.4} \\
Semantic indexing & \texttt{text-embedding-3-small} \\
\bottomrule
\end{tabular}
}
\caption{\textbf{API assignment.} GPT-5.4 mini is used for L0 memory construction, while GPT-5.4 handles higher-level construction(L1 L2 and L3), retrieval, answering, and evaluation.}
\label{tab:implementation}

\end{table}

\subsection{Hierarchy Footprint}

Table~\ref{tab:hierarchy} reports the realized hierarchy size at each corpus tier, including the complete 618M-token collection and the Expanded 1.14B-token collection. The number of L0 nodes increases with corpus scale, while the dynamic construction procedure determines the number of L1 and L2 nodes for each collection. At 1.14B tokens, the hierarchy contains 865,120 source-level nodes but only 930 L1 nodes and 31 L2 nodes, preserving a compact routing structure as the source collection grows.

\begin{table}[!ht]
\centering
\footnotesize
\resizebox{0.8\columnwidth}{!}{%
\begin{tabular}{lrrr}
\toprule
Tier & L0 nodes & L1 nodes & L2 nodes \\
\midrule
10M  & 8,927   & 94  & 9 \\
20M  & 17,051  & 287 & 9 \\
60M  & 49,978  & 223 & 14 \\
100M & 82,779  & 287 & 16 \\
150M & 124,330 & 352 & 18 \\
250M & 207,179 & 455 & 21 \\
FULL (618M) & 511,962 & 716 & 27 \\
Expanded (1.14B) & 865,120 & 930 & 31 \\
\bottomrule
\end{tabular}
}
\caption{\textbf{Realized hierarchy sizes.} The dynamic construction procedure determines the number of nodes at each level for every corpus tier. FULL is the unmodified 618M-token benchmark collection; Expanded is the merged 1.14B-token collection described in Section~\ref{sec:data_extension}.}
\label{tab:hierarchy}

\end{table}

\section{Promotion-Floor Sensitivity}
\label{sec:promotion_floor}

We control promotion through a similarity floor \(\theta\): a candidate node \(v\) is promoted only when its promotion score satisfies
\[
g_t(v)\ge\theta,
\]
where \(g_t(v)\in[0,1]\) measures its similarity to the current query and usage state. If the floor is too low, many weakly related nodes are promoted, which introduces noisy priority signals and can displace more useful candidates. If the floor is too high, few nodes are promoted, limiting the system's ability to adapt retrieval priority across queries. Table~\ref{tab:promote_relevance_floor}, together with Figures~\ref{fig:floor10}, \ref{fig:floor150}, and~\ref{fig:floorfull}, shows that the response to the floor depends on collection size. The best tested floors are 0.9 at 10M, 0.3 at 150M, and 0.4 on FULL (618M), reaching 82.75\%, 76.15\%, and 70.24\% Combined, respectively. These maxima match the corresponding results in Table~\ref{tab:scaling}; the sweeps do not support a shared optimal floor across the three scales.

\begin{table}[!ht]
\centering
\footnotesize
\setlength{\tabcolsep}{4pt}
\begin{tabular}{rrrr}
\toprule
Floor & 10M & 150M & FULL (618M) \\
\midrule
0.1 & 77.86 & 72.89 & 68.69 \\
0.2 & 75.35 & 70.99 & 67.89 \\
0.3 & 79.58 & \textbf{76.15} & 66.04 \\
0.4 & 80.35 & 74.22 & \textbf{70.24} \\
0.5 & 74.95 & 75.16 & 67.96 \\
0.6 & 78.71 & 73.48 & 67.10 \\
0.7 & 81.37 & 68.20 & 67.66 \\
0.8 & 81.51 & 74.15 & 67.62 \\
0.9 & \textbf{82.75} & 72.73 & 70.09 \\
\bottomrule
\end{tabular}
\caption{\textbf{Promotion-floor sensitivity.} Combined (\%) is reported at each floor for three collection sizes. The best scores occur at 0.9 for 10M, 0.3 for 150M, and 0.4 for FULL (618M); bold marks the maximum in each column.}
\label{tab:promote_relevance_floor}
\label{tab:floor10}
\label{tab:floor150}
\label{tab:floorfull}

\end{table}

\begin{figure}[!ht]
  \centering
  \includegraphics[width=0.8\columnwidth]{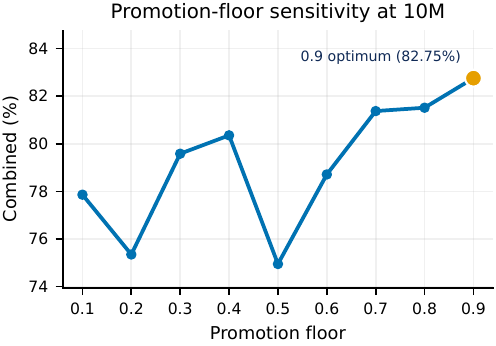}
  \caption{\textbf{Promotion-floor sensitivity at 10M.} The highest tested Combined score is 82.75\% at a floor of 0.9.}
  \label{fig:floor10}

\end{figure}

\begin{figure}[!ht]
  \centering
  \includegraphics[width=0.8\columnwidth]{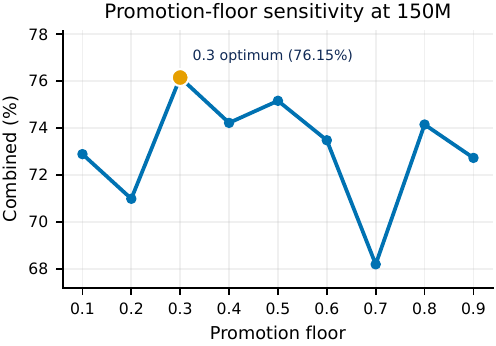}
  \caption{\textbf{Promotion-floor sensitivity at 150M.} The highest tested Combined score is 76.15\% at a floor of 0.3.}
  \label{fig:floor150}

\end{figure}

\begin{figure}[!ht]
  \centering
  \includegraphics[width=0.85\columnwidth]{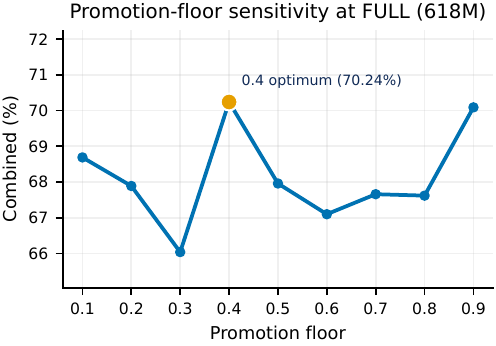}
  \caption{\textbf{Promotion-floor sensitivity on the complete 618M-token collection.} The highest tested Combined score is 70.24\% at a floor of 0.4.}
  \label{fig:floorfull}

\end{figure}

\section{Detailed Algorithms}
\label{app:mechanisms}

This appendix provides the full procedures for the three mechanisms introduced in Section~\ref{sec:system}: dynamic hierarchy construction, dual-memory retrieval with selective detailed-memory loading, and on-demand promotion with demotion.

\paragraph{Dynamic hierarchy construction.} Algorithm~\ref{alg:hierarchy} starts from the L0 source nodes and iteratively proposes a higher abstraction level. At each step, the controller evaluates the proposed grouping and either accepts it, requests a revised grouping, or stops construction. The process therefore determines both the structure and depth of the hierarchy from the current collection rather than fixing them in advance.

\begin{algorithm}[!ht]
\caption{Dynamic multi-level hierarchy construction}
\label{alg:hierarchy}
\small
\begin{algorithmic}[1]
\Require L0 nodes \(\mathcal{V}_0\), maximum depth \(D\), maximum proposals \(A\)
\State \(\mathcal{T}\gets\mathcal{V}_0;\ \ell\gets0;\ a\gets0\)
\While{\(\ell<D\) and \(a<A\)}
  \State \(\mathcal{C}\gets\Call{ProposeGroups}{\mathrm{dist}(\mathcal{V}_\ell)};\ a\gets a+1\)
  \State \(s\gets\Call{ProposalStats}{\mathcal{C}}\)
  \State \(u\gets\Call{Controller}{s}\)
  \If{\(u=\textsc{Stop}\)}
    \State \textbf{break}
  \EndIf
  \If{\(u=\textsc{Recluster}\)}
    \State revise the grouping proposal
    \State \textbf{continue}
  \EndIf
  \State \(\mathcal{V}_{\ell+1}\gets\Call{BuildParentRecords}{\mathcal{C}}\)
  \For{\(p\in\mathcal{V}_{\ell+1}\)}
    \State \(\mathrm{children}(p)\gets\{v\in\mathcal{V}_\ell:v\text{ belongs to }p\}\)
  \EndFor
  \State add \(\mathcal{V}_{\ell+1}\) and their edges to \(\mathcal{T}\)
  \State \(\ell\gets\ell+1\)
\EndWhile
\State \Return \(\mathcal{T}\)
\end{algorithmic}
\end{algorithm}

\paragraph{Dual-memory retrieval and selective detailed-memory loading.} Algorithm~\ref{alg:selection} performs coarse-to-fine retrieval over distilled memory until it reaches L0. Signals from different hierarchy paths are merged for each source candidate, after which only the highest-ranked detailed memories that fit the answer-input budget are loaded. This procedure uses compact memory for retrieval while reserving the answer context for selected detailed evidence.

\begin{algorithm}[!ht]
\caption{Dual-memory retrieval and selective detailed-memory loading}
\label{alg:selection}
\small
\begin{algorithmic}[1]
\Require query \(q\), hierarchy \(\mathcal{T}\), answer-input budget \(B\)
\State \(\mathcal{V}\gets\Call{HighestLevel}{\mathcal{T}}\)
\While{\(\mathcal{V}\not\subseteq L_0\)}
  \State \(\mathcal{S}\gets\Call{RetrieveDistilled}{q,\mathcal{V}}\)
  \State \(\mathcal{V}\gets\bigcup_{v\in\mathcal{S}}\mathrm{children}(v)\)
\EndWhile
\State \(\mathcal{R}\gets\Call{MergeRouteSignals}{\mathcal{V}}\)
\State \(\mathcal{C}\gets\Call{RankCandidates}{\mathcal{R}}\)
\State \(X\gets\varnothing;\ t_{\mathrm{used}}\gets\Call{FixedAnswerCost}{q}\)
\For{\(v\in\mathcal{C}\)}
  \If{\(t_{\mathrm{used}}+\tau_v^d\le B\)}
    \State append \(d_v\) to \(X\)
    \State \(t_{\mathrm{used}}\gets t_{\mathrm{used}}+\tau_v^d\)
  \EndIf
\EndFor
\State \((a,C_{\mathrm{pred}})\gets\Call{Answer}{q,X}\)
\State \(C_{\mathrm{pred}}\gets C_{\mathrm{pred}}\cap\mathrm{ID}(X)\)
\State \Return \((a,C_{\mathrm{pred}},X)\)
\end{algorithmic}
\end{algorithm}

\paragraph{On-demand promotion and demotion.} Algorithm~\ref{alg:promotion} updates node priority from query-time evidence and prior usage. Nodes above the promotion threshold receive a current-query score increase and enter the cross-query active state. Repeated use refreshes their retention, while expired or low-retention nodes are demoted when the active-state budget is exceeded.

\begin{algorithm}[!ht]
\caption{On-demand memory promotion and demotion}
\label{alg:promotion}
\small
\begin{algorithmic}[1]
\Require query \(q_t\), candidates \(\mathcal{C}_t\), active state \(\mathbf{z}_t\), threshold \(\theta\), active-state budget \(K\)
\For{\(v\in\Call{Deduplicate}{\mathcal{C}_t}\)}
  \State \(g_t(v)\gets\Call{PromotionScore}{q_t,v,\mathbf{z}_t}\)
  \If{\(g_t(v)\ge\theta\)}
    \State increase the current ranking score of \(v\)
    \State upsert \((v,g_t(v),t)\) into \(\mathbf{z}_t\)
    \State append \textsc{Promote}\((v)\) to the transition log
  \ElsIf{\(v\in\operatorname{active}(\mathbf{z}_t)\)}
    \State refresh the score and last-use time of \(v\)
  \EndIf
\EndFor
\For{\(v\in\operatorname{active}(\mathbf{z}_t)\)}
  \State \(\rho_t(v)\gets\max\left(0,s_t(v)-\frac{t-t_{\mathrm{last}}(v)}{\tau}\right)\) (Equation~\ref{eq:retention})
\EndFor
\While{\(|\operatorname{active}(\mathbf{z}_t)|>K\) or an active node has expired}
  \State demote the expired or minimum-retention node
  \State append \textsc{Demote} to the transition log
\EndWhile
\State \Return updated ranking scores and \(\mathbf{z}_{t+1}\)
\end{algorithmic}
\end{algorithm}

\section{System Properties}
\label{app:properties}

This section states several properties of the \sys execution process. These results describe retrieval coverage, candidate use, token accounting, and deferred preparation; they do not guarantee answer correctness.

\paragraph{Complementary retrieval routes.} Let \(H_d\) denote the event that a gold source is recovered through direct L0 retrieval, and let \(H_h\) denote recovery through hierarchical retrieval over distilled memory. The recall of their union is
\[
R_{\mathrm{fuse}}=\Pr(H_d\cup H_h)=R_d+\Pr(H_h\cap H_d^c).
\]
The hierarchical route therefore improves recall only when it recovers a source missed by direct L0 retrieval. This result explains why \sys combines the two routes rather than replacing direct retrieval with hierarchy traversal.

Let \(L=(h_1,h_2,\ldots)\) be a ranked list of retrieval hits, and let \(\sigma(h)\) denote the L0 source associated with hit \(h\). We compare two procedures under a source budget \(k\): taking the first \(k\) hits and then removing duplicates, or scanning \(L\) until \(k\) distinct sources are collected.

\begin{proposition}[Distinct-source coverage]
\label{thm:dedup}
If \(L\) contains at least \(k\) distinct sources, scanning until \(k\) distinct sources are collected returns a superset of the sources obtained by deduplicating the first \(k\) hits.
\end{proposition}

\begin{proof}
Every source appearing among the first \(k\) hits is encountered before the distinct-source scan terminates. Duplicate hits allow the scan to continue to later positions and include additional sources. Therefore, distinct-source collection cannot reduce source coverage under the same budget.
\end{proof}

\paragraph{Answer-input accounting.} Let \(E_q\) be the set of L0 detailed memories loaded for query \(q\), and let \(T_0(q)\) denote the fixed token cost of the question and instructions. The total answer-input cost is
\[
T(q)=T_0(q)+\sum_{v\in E_q}\tau_v^d\le B.
\]
If a detailed-only policy would load a candidate set \(\mathcal{C}_q\), selective detailed-memory loading saves $\sum_{(v\in\mathcal{C}_q\setminus E_q)} \tau_v^d$ tokens. Distilled-memory tokens are used during retrieval and do not enter the answer-input budget.

\paragraph{Deferred-preparation cost.} Let \(c_v\) be the cost of preparing a representation for node \(v\), and let \(T_v\) be the first query step at which that representation is needed. Over a query horizon \(H\), the expected on-demand preparation cost is:
\begin{equation}
\begin{aligned}
\mathbb{E}[C_{\mathrm{on\mbox{-}demand}}(H)]&=\sum_v c_v\Pr(T_v\le H)\\&\le\sum_v c_v=C_{\mathrm{eager}}.
\end{aligned}
\end{equation}
Thus, preparing a representation only when it is first needed cannot cost more than preparing every representation in advance over the same horizon. This is a cost-accounting result rather than an end-to-end latency guarantee, and it applies only to representations whose construction is actually deferred.

\section{Additional Scaling and Component Results}
\label{app:additional_results}

\subsection{Scaling of Direct Source-Level Retrieval}
\label{app:leafscale}

Table~\ref{tab:leafscale} reports the performance of global source-leaf search without promotion across increasing collection sizes. Document Recall and Correct generally decrease as the collection grows, showing the limits of direct source-level retrieval without adaptive hierarchy navigation and promotion.

\begin{table}[!ht]
\centering
\small
\begin{tabular}{lrr}
\toprule
Scale & DocRcl (\%) \(\uparrow\) & Correct (\%) \(\uparrow\) \\
\midrule
10M & 68.25 & 81.2 \\
20M & 65.19 & 80.2 \\
60M & 55.02 & 74.2 \\
100M & 59.04 & 74.4 \\
150M & 55.30 & 77.2 \\
\bottomrule
\end{tabular}
\caption{\textbf{Global source-leaf retrieval across corpus scales.} Document Recall and Correct are reported as percentages.}
\label{tab:leafscale}

\end{table}

\subsection{Online Token Use}
\label{app:runtime_tokens}

Table~\ref{tab:runtime_tokens} decomposes online token use at 20M into retrieval-control and answer-generation costs. Mean is the average number of tokens per query, Median is the 50th-percentile value, and p95 is the 95th-percentile value, meaning that 95\% of queries use no more than this number of tokens. Answer input is the largest component, while retrieval-control output and answer output account for only a small share of the total.

\begin{table}[!ht]
\centering
\small
\begin{tabular}{lrrr}
\toprule
Runtime component & Mean & Median & p95 \\
\midrule
Retrieval-control input & 3,758 & 3,969 & 4,270 \\
Retrieval-control output & 155 & 155 & 186 \\
Answer input & 13,629 & 13,284 & 19,413 \\
Answer output & 142 & 124 & 291 \\
\midrule
Total online tokens & 17,683 & 17,425 & 23,534 \\
\bottomrule
\end{tabular}
\caption{\textbf{Online token use at 20M.} Columns report the mean, median, and 95th-percentile token count per query. Total tokens include retrieval-control input and output and answer input and output.}
\label{tab:runtime_tokens}

\end{table}

\subsection{External visualization Results}
\label{app:external_availability}

Figure~\ref{fig:scaling} illustrates how retrieval quality and memory activity evolve with corpus scale. \sys consistently exceeds the published LB\#1 baseline across all official tiers, maintaining its advantage even at 618M tokens. Although retrieval becomes more challenging as collections grow, adaptive promotion and demotion remain active, supporting efficient memory reuse. The 1.14B-token extension further demonstrates scalability beyond the official benchmark, with a moderate decline in answer quality.

\begin{figure}[!ht]
  \centering
  \includegraphics[width=\columnwidth]{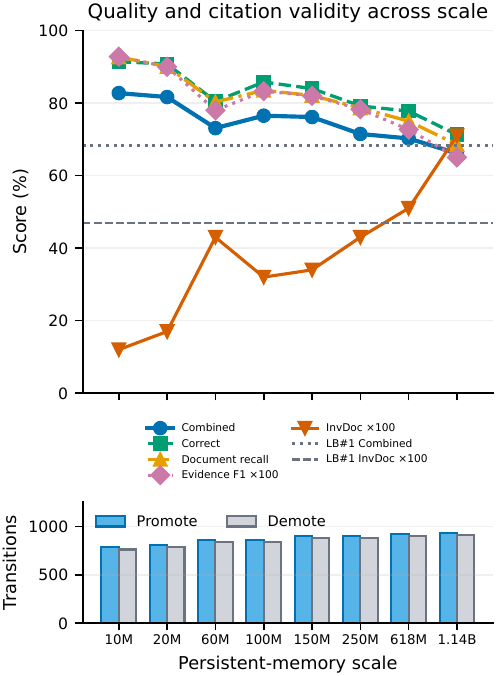}
  \caption{\textbf{Scaling performance and memory activity.} \sys maintains a Combined advantage over LB\#1 across all official corpus tiers, including FULL (618M tokens). The Expanded collection (1.14B tokens) further evaluates billion-token scalability, while promotion and demotion track adaptive memory management.}
  \label{fig:scaling}
  
\end{figure}

Figure~\ref{fig:frontier} illustrates the trade-off between answer quality and context consumption at 20M tokens. \sys achieves the highest Combined score while substantially reducing answer-input tokens compared with detailed-only loading. Distilled-only memory minimizes token usage but sacrifices accuracy, whereas loading both representations increases context cost without improving quality. These results highlight the effectiveness of selective evidence loading in balancing retrieval efficiency and answer quality.

\begin{figure}[!ht]
  \centering
  \includegraphics[width=0.9\columnwidth]{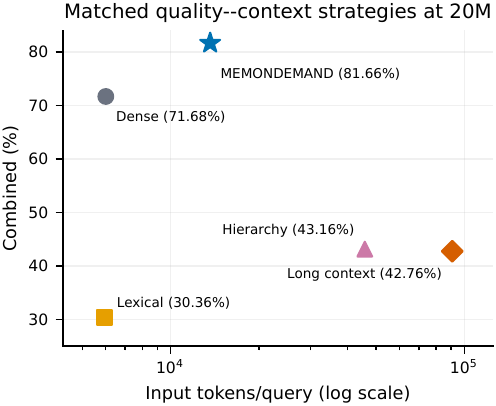}
  \caption{\textbf{Quality--context configurations at 20M.}
  Complete strategies use their own evidence-loading policies; the horizontal axis is
  logarithmic.}
  \label{fig:frontier}

\end{figure}

Figure~\ref{fig:packing} compares answer quality with answer-input cost across evidence-loading policies. At 10M and 20M, selective detailed-memory loading achieves 82.75\% and 81.66\% Combined with 16,944 and 13,629 answer-input tokens per query, respectively. Relative to detailed-only loading, these policies reduce answer-input tokens by 63.9\% and 70.2\%, while detailed-only Combined is 82.68\% and 77.91\%. Loading both representations reaches 70.60\% and 68.22\%, and distilled-only loading reaches 51.08\% and 49.76\%.

\begin{figure}[!ht]
  \centering
  \includegraphics[width=\columnwidth]{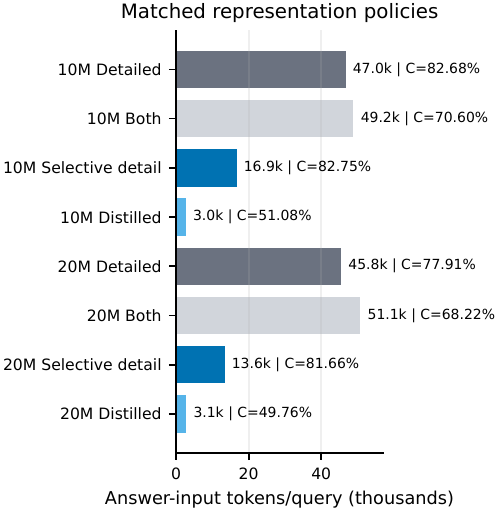}
  \caption{\textbf{Quality and answer-input cost.}
  Labels report Combined for the evidence-loading policies in Table~\ref{tab:packing}; selective loading reaches 82.75\% at 10M and 81.66\% at 20M.}
  \label{fig:packing}

\end{figure}



Figure~\ref{fig:ablation} visualizes the degradation caused by removing or restricting individual components at 10M. The full system reaches 82.75\% Combined. One-step navigation, no promotion, and a 25-candidate pool reach 64.52\%, 64.20\%, and 62.64\%, with absolute score differences of -18.23\%, -18.55\%, and -20.11\% from the full system, respectively. These results show that multi-step navigation, on-demand promotion, and controlled candidate selection contribute jointly to final quality.

\begin{figure}[!ht]
  \centering
  \includegraphics[width=\columnwidth]{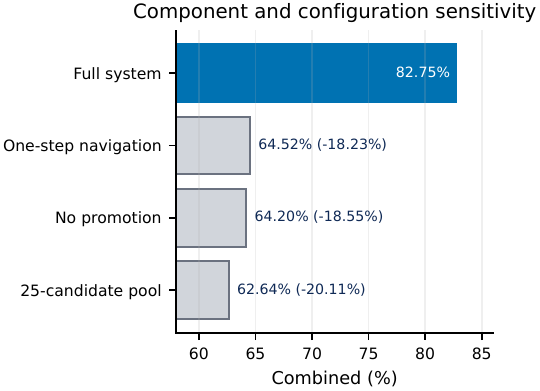}
  \caption{\textbf{Component and configuration sensitivity.}
  Full-system Combined is 82.75\%; the three variants reach 64.52\%, 64.20\%, and 62.64\%. Parenthesized labels show absolute score differences from the full system. Promotion removal disables both the current-query bonus and the managed-state update.}
  \label{fig:ablation}

\end{figure}




Figure~\ref{fig:availability} compares the up-front wall time of on-demand promotion with eager preparation of all parent-level representations at 20M. The measured on-demand preparation time is 1.46 seconds, compared with an estimated 593.5 seconds for eager preparation, a \(407\times\) ratio. The large gap shows that on-demand promotion avoids preparing memory that may never be reused, substantially reducing the cost required before query execution.

\begin{figure}[!ht]
  \centering
  \includegraphics[width=0.7\columnwidth]{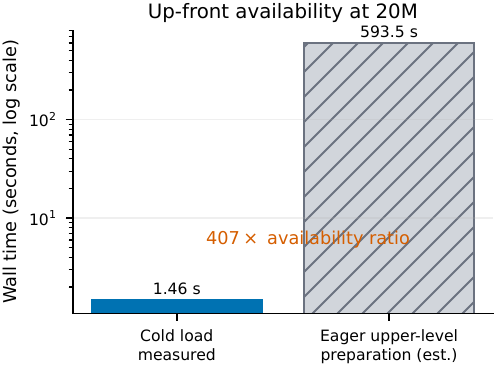}
  \caption{\textbf{Up-front availability at 20M.} On-demand preparation takes 1.46 seconds; eager parent preparation is estimated at 593.5 seconds, a \(407\times\) ratio.}
  \label{fig:availability}

\end{figure}

Figure~\ref{fig:external_profile} summarizes retrieval and answering performance on FinanceBench, HotpotQA, and FRAMES. Document Recall is 80.00\%, 49.90\%, and 88.76\%; Evidence F1 is 0.729, 0.549, and 0.617; and Correct is 66.0\%, 78.4\%, and 55.0\%, respectively, matching Table~\ref{tab:external}. The figure displays Evidence F1 multiplied by 100. These results show that source recovery and answer correctness vary across tasks, reflecting their different evidence structures and reasoning requirements.

\begin{figure}[!ht]
  \centering
  \includegraphics[width=0.8\columnwidth]{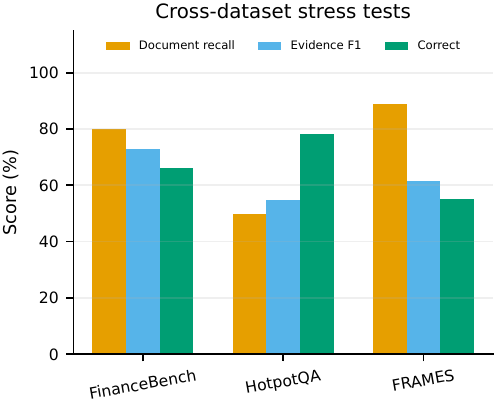}
  \caption{\textbf{External benchmark profile.} Document Recall is 80.00\%, 49.90\%, and 88.76\% on FinanceBench, HotpotQA, and FRAMES, respectively. Evidence F1 is multiplied by 100 for display; all values match Table~\ref{tab:external}.}
  \label{fig:external_profile}

\end{figure}

\section{Prompt Templates}
\label{app:prompts}

This section provides the prompt templates used for memory construction, hierarchy control, retrieval, promotion, answer generation, and evaluation. The final answer template is transcribed from \texttt{improved\_answer\_prompt.py}, with only TeX escaping and line wrapping changed for presentation.

\begin{promptbox}[teal]{Source-leaf distilled-record construction}
\textbf{System:} Rewrite the source into a distilled routing record containing topic,
entities, dates, identifiers, and decision-bearing facts.  Preserve negation and qualifiers.
Do not add outside knowledge.  Return JSON with \texttt{distilled\_text},
\texttt{key\_facts}, and \texttt{source\_id}.

\textbf{User:} One source chunk with tenant metadata, section path, and source ID.
\end{promptbox}

\begin{promptbox}[blue]{Parent-abstraction decision}
\textbf{System:} Inspect coherence, residual noise, group sizes, representative content, and
remaining budget.  Return exactly \texttt{CREATE}, \texttt{RECLUSTER}, or \texttt{STOP}, plus
a reason.  Never alter source IDs or descendant links.

\textbf{User:} Proposal statistics and representative distilled routing records.
\end{promptbox}

\begin{promptbox}[orange]{Parent-abstraction construction}
\textbf{System:} Summarize the child records for routing.  Preserve entities, time ranges,
constraints, exceptions, and disagreements.  Return the unchanged descendant source IDs and
introduce no facts absent from the children.

\textbf{User:} A bounded cluster of distilled routing records and source IDs.
\end{promptbox}

\begin{promptbox}[violet]{Navigation and source selection}
\textbf{System:} Use distilled routing records only for search.  Return source IDs from the supplied
candidate set and stop when the evidence need is resolved or the navigation budget is spent.

\textbf{User:} Question, hierarchy frontier, route scores, and managed-state metadata.
\end{promptbox}

\begin{promptbox}[red]{Promotion gate}
\textbf{System:} Return each supplied source ID, promotion decision, and score.  Promotion
may change current priority and later reuse but does not load a detailed evidence payload.

\textbf{User:} Question, unique-source candidates, route signals, state, and capacity.
\end{promptbox}

\begin{promptbox}[green!60!black]{Evidence-grounded answer}
\textbf{System:} You are an enterprise memory question-answering assistant.

You will be given: (1) a user \texttt{QUERY}; (2) \texttt{[CTX]} hierarchy context blocks for
background scope, which must not be cited; and (3) \texttt{[EVID]} evidence blocks, each with a
\texttt{node\_id}. Some evidence blocks contain only a \texttt{[BRIEF]} distilled summary;
top-ranked evidence also contains the complete \texttt{[FULL]} source text.

\textbf{Critical anti-hallucination rules:}
\begin{itemize}[leftmargin=*,nosep]
  \item Every specific value stated---a number, threshold, date, flag name, configuration key,
  dollar amount, percentage, duration, person name, or other concrete detail---must be an exact
  copy of text that literally appears in an \texttt{[EVID]} block. Do not paraphrase numbers,
  infer a value, or estimate one even if it appears reasonable.
  \item Before writing any specific value, silently verify: ``Can I point to the exact substring
  in \texttt{[EVID]} that contains this value?'' If not, omit the value or state that the
  document does not specify it.
  \item If part of the answer is supported but another part requires guessing a specific value,
  answer only the supported part and explicitly note what is not specified.
  \item Never fill in a plausible number, date, or threshold merely to be helpful. An incomplete
  but accurate answer is preferable to a complete but partially fabricated answer.
  \item Inspect conversational or reply sections in \texttt{[FULL]} text, including review
  conversations, comment threads, and follow-up messages. These may contain the controlling
  configuration key or value rather than the initial description.
\end{itemize}

\textbf{Rules:}
\begin{itemize}[leftmargin=*,nosep]
  \item Write a precise, factual answer of one to five sentences using only information from
  \texttt{[EVID]} blocks.
  \item Ground the answer primarily in \texttt{[FULL]} source text. Treat \texttt{[BRIEF]} as a
  topic hint only. Write one to three precise sentences, and make every factual claim traceable
  to a specific \texttt{[FULL]} passage.
  \item After the answer, on a new line, output \texttt{CITED: id1,id2,...}.
  \item Cite every \texttt{[EVID]} node whose \texttt{[FULL]} or \texttt{[BRIEF]} content was
  used. If three or more nodes contributed, cite all of them. Do not cite \texttt{[CTX]}
  node IDs.
  \item If no \texttt{[EVID]} context applies, output \texttt{INSUFFICIENT EVIDENCE} and an empty
  \texttt{CITED:} line. Never invent facts absent from the provided context.
\end{itemize}

\textbf{User:} Query, hierarchy context blocks, and evidence blocks with source IDs.
\end{promptbox}

\begin{promptbox}[magenta]{Benchmark evaluation}
\textbf{System:} Compare the candidate and canonical answers.  Return schema-constrained JSON
with \texttt{correct}, \texttt{validated\_facts}, and \texttt{missing\_facts}.

\textbf{User:} Question, candidate answer, and canonical answer.
\end{promptbox}

\end{document}